\documentclass{article}
\usepackage[utf8]{inputenc}
\usepackage[T1]{fontenc}
\usepackage{ijcai26}

\usepackage{times}
\usepackage{soul}
\usepackage{url}
\usepackage[hidelinks]{hyperref}
\usepackage[small]{caption}
\usepackage{graphicx}
\usepackage{amsmath}
\usepackage{amsthm}
\usepackage{booktabs}
\usepackage[usenames,dvipsnames]{xcolor}

\usepackage{amsmath}
\usepackage{amsthm}
\usepackage{amssymb}
\usepackage{graphicx}
\usepackage{float}
\usepackage{enumerate}
\usepackage{multicol}
\usepackage{csquotes}
\usepackage{subcaption}
\usepackage{xparse}
\usepackage{tikz}
\usepackage{url}
\usepackage{array}
\usepackage{etoolbox}
\usepackage{nicematrix}

\usepackage{graphicx}
\usepackage{xspace}
\usepackage{xparse}
\usepackage{enumerate}
\usepackage{tabularx}
\usepackage{booktabs}
\usepackage{csquotes}
\usepackage{latexsym}
\usepackage{wasysym}
\usepackage{multicol}
\usepackage{multirow}
\usepackage{subcaption}
\usepackage{arydshln}
\usepackage{rotating}
\usepackage{stmaryrd}

\usepackage{hyperref}
\hypersetup{hidelinks}
\usepackage{relsize}
\usepackage[export]{adjustbox}

\usepackage{microtype}
\usepackage[noend,ruled,vlined]{algorithm2e}

\usepackage{tikz}
\usetikzlibrary{fadings}
\usetikzlibrary{shapes,decorations,positioning}
\usetikzlibrary{fadings,shapes,decorations,positioning,calc}
\usetikzlibrary{decorations.markings,decorations.pathreplacing}
\usetikzlibrary{shapes.geometric,automata,positioning}
\usetikzlibrary{shadows}\usetikzlibrary{calc}
\usetikzlibrary{decorations,decorations.pathmorphing,decorations.pathreplacing}
\usetikzlibrary{calc}
\usetikzlibrary{tikzmark}

\newcolumntype{L}[1]{>{\raggedright\let\newline\\\arraybackslash\hspace{0pt}}m{#1}}
\newcolumntype{C}[1]{>{\centering\let\newline\\\arraybackslash\hspace{0pt}}m{#1}}
\newcolumntype{R}[1]{>{\raggedleft\let\newline\\\arraybackslash\hspace{0pt}}m{#1}}

\newcommand{\tuple}[1]{\ensuremath{\langle #1\rangle}}

\DeclareDocumentCommand{\todo}{o g}{\IfNoValueTF{#1}{\begingroup\color{magenta}TODO: #2\endgroup}{\begingroup\color{magenta}#1 #2\endgroup}}

\newcommand{\acr}[1]{\textsf{\textsmaller[0.8]{#1}}}
\newcommand{\Paths}[3]{\mathcal{R}_{#3}^{#2}(#1)}
\newcommand{\Dis}{\mathit{Dis}}
\newcommand{\dbs}{\rho^\text{dbs}}

\newcommand{\vertexa}{v}
\newcommand{\vertexb}{v'}

\newcommand{\KSendexample}{}

\newcommand{\problemDisEquivalence}{\textsc{EquivDis}}
\newcommand{\problemDisStronger}{\textsc{StrongerDis}}
\newcommand{\problemGraphPathLength}{\textsc{EqualWalkCount}}
\newcommand{\problemAutomataEquivalence}{\textsc{\( \mathbb{Q} \)-Equivalence}}
\newcommand{\problemAutomataEquivalenceZ}{\textsc{\( \mathbb{Z} \)-Equivalence}}

\newcommand{\forwardspace}{\ensuremath{\mathcal{F}_{\!\!\mathcal{A}}}}
\newcommand{\baseofA}{\ensuremath{\mathcal{B}_{\!\mathcal{A}}}}

\newcommand{\KSciteauthoryear}[1]{\citeauthor{#1}~(\citeyear{#1})}

\newtheorem{theorem}{Theorem}[section]
\newtheorem{proposition}[theorem]{Proposition}
\newtheorem{lemma}[theorem]{Lemma}

\theoremstyle{definition}
\newtheorem{definition}[theorem]{Definition}
\newtheorem{example}[theorem]{Example}

\newtheorem*{observation*}{Observation}

\newcounter{afexample}
\usepackage{tikz}
\usetikzlibrary{arrows.meta, arrows}

\title{On the Complexity of the Discussion-based Semantics\\ in
Abstract Argumentation}

\author{
	Lydia Blümel$^1$ \and
     Kai Sauerwald$^1$\and
    Kenneth Skiba$^{1,2}$\And
    Matthias Thimm$^1$\\
    \affiliations
    $^1$University of Hagen \quad $^2$ University of Luxembourg
}

\begin{document}
    
\maketitle 
\pagestyle{plain} %

\begin{abstract}
We show that deciding whether an argument $a$ is stronger than an argument $b$ with respect to the discussion-based semantics of Amgoud and Ben-Naim is decidable in polynomial time. At its core, this problem is about deciding whether, for two vertices in a graph, the number of walks of each length ending in those vertices is the same. We employ results from automata theory and reduce this problem to the equivalence problem for semiring automata. This offers a new perspective on the computational complexity of ranking semantics, an area in which the complexity of many semantics remains open.
\end{abstract}

\section{Introduction}
Abstract argumentation frameworks (\acr{AAFs}) \cite{DBLP:journals/ai/Dung95} are an approach for modelling argumentative scenarios that represents arguments as vertices in a directed graph, where a directed edge represents an attack from one argument to another. Besides the classical approach for reasoning in \acr{AAFs} consisting of identifying sets of arguments (\emph{extensions}) that form a \emph{plausible} outcome of the argumentation, another important approach are \emph{ranking-based} semantics \cite{DBLP:conf/sum/AmgoudB13a,Bonzon:2016a,Bonzon:2023}. A ranking-based semantics gives a ranking (usually a total preorder) over the set of arguments of an \acr{AAF}, with the intended meaning that higher ranks indicate stronger acceptability or plausibility. Using ranking-based semantics, a more finer-grained assessment of the strength of arguments can be achieved, in contrast to classical extension-based semantics. Several different approaches for ranking-based semantics (and the related \emph{gradual semantics}) have been proposed over the years \cite{DBLP:conf/ijcai/0007PT24,Potyka:2019a,DBLP:journals/ijar/BaroniRT19,Grossi:2019,Thimm:2018a,Blumel:2022,Skiba:2021w,Skiba:2020a,Mailly:2020a,Bonzon:2018,Dondio:2018,Yun:2018a,Bonzon:2016,DBLP:conf/sum/AmgoudB13a,Cayrol:2005}, ranging from approaches founded in the topological analysis of the underlying graph structure \cite{DBLP:conf/sum/AmgoudB13a,Dondio:2018} to approaches generalising the extension-based perspective to rankings \cite{Blumel:2022}, see also \cite{Bonzon:2023} for a recent survey.

In this paper, we are addressing issues of computability and complexity of ranking-based semantics. Interestingly, very little work has been done on this aspect so far. Some works such as \cite{Delobelle:2017,Bonzon:2023} do present algorithms for some ranking-based semantics and conduct empirical evaluations, but the exact computational complexity of the even most basic ranking-based semantics from \cite{DBLP:conf/sum/AmgoudB13a} has not been settled until now. This is not so surprising, after all, the problem of deciding whether some argument $a$ is stronger than another argument $b$ wrt., in particular, the \emph{discussion-based semantics} \cite{DBLP:conf/sum/AmgoudB13a}, has some underlying combinatorial challenges when it comes to considering walks of infinite lengths in the argumentation framework. In fact, when abstracting from the application of ranking-based semantics in abstract argumentation, the underlying graph theoretical problem is quite intricate and has, to the best of our knowledge, not been considered in its generality so far. 

We consider the \emph{discussion-based semantics} \cite{DBLP:conf/sum/AmgoudB13a} in this paper as an example for a simple ranking-based semantics. We will present the formal background in Section~\ref{sec:aaf}, but this semantics, roughly, assesses an argument $a$ is stronger than an argument $b$ if the number of odd (resp.\ even) walks ending in $a$ are less (resp. more) than the corresponding number of walks ending in $b$. Implementations of this semantics, such as \cite{Bonzon:2023}, usually set some maximal length of walks to consider in order to have a terminating algorithm. However, no formal justification is known why any maximal length maintains the correctness of the algorithm. We will settle this question and show, in particular, that the problem of deciding whether an argument $a$ is stronger than an argument $b$ wrt.\ the discussion-based semantics is not only decidable but also decidable in polynomial time. In order to show this, we use results from automata theory and, in particular, about the equivalence of \emph{automata over semirings} \cite{KS_Tzeng1996,KS_BealLombardySakarovitch2005}. Our result does not only shed light on computational questions regarding the discussion-based semantics, but also on all ranking-based semantics that are defined in terms of walks in an abstract argumentation framework. Moreover, our result also solves an interesting general graph-theoretical problem that has not been considered until now, i.\,e., deciding for two given vertices whether, for each length, the number of walks of that length is the same.

\noindent The main contributions of this paper are:
\begin{itemize}    
    \item {[Interlinking Concepts]} We introduce the notion of \emph{agreement on the numbers of walks} for two vertices of a graph, which is for each length the number of walks of that length that end in these vertices is the same. This notion offers the following connections between areas:
    \begin{itemize}
        \item {[Ranking Semantics \( \to \) Graph Theory]} The agreement on the numbers of walks strongly corresponds to the equivalence of strength of arguments with respect to discussion-based semantics.
        \hfill{\color{gray}(Section~\ref{sec:graphperspective})}
        
        \item {[Graph Theory \( \to \) Semiring Automata]}  For two vertices \(v\) and \(u\), there are two semiring automata over the rationals that are equivalent if and only if \(v\) and \(u\) agree on the number of walks. \hfill{\color{gray}(Section~\ref{sec:Graph2Semiring})}
    \end{itemize}
    \item {[Complexity of {\problemGraphPathLength}]} Deciding if two vertices of a graph agree on the numbers of walks is shown to be in \textbf{PTIME}. 
    \hfill{\color{gray}(Section~\ref{sec:Graph2Semiring})}
    \item {[Complexity of {Discussion-Based Semantics}]} Deciding {\problemDisStronger}, i.\,e, if an argument is stronger than another argument with respect to discussion-based semantics, and deciding {\problemDisEquivalence}, i.\,e.,  if two arguments have the same strength with respect to discussion-based semantics, are shown to be in \textbf{PTIME}. 
    \hfill{\color{gray}(Section~\ref{sec:ComplexityDBS})}
\end{itemize}
Notably, we will derive that the length of walks one must consider in each of the aforementioned problems is bounded. This paper is the first to establish such bounds, which are central to identifying the computational complexity.

The paper is organized as follows. The main results are presented in Section~\ref{sec:graphperspective}, Section~\ref{sec:Graph2Semiring} and Section~\ref{sec:ComplexityDBS} as indicated above.
We are presenting in Section~\ref{sec:basicnotions} (basic notions), Section~\ref{sec:aaf} (abstract argumentation), and Section~\ref{sec:semiringautomatEQ} (semiring automata) the required background and notions.
We conclude in Section~\ref{sec:conclusion}.

\section{Background and Basic Notions}
\label{sec:basicnotions}

With \( \mathbb{N}=\{1,2,3,\ldots \} \)  we denote the set of positive integers
and 
denote by \( \mathbb{N}_0=\mathbb{N}\cup\{0\} \) the set of non-negative integers.
The semiring of rational numbers is the tuple \( \tuple{\mathbb{Q},\cdot,+,0,1} \), where the product \enquote{\( \cdot \)} of the semiring is the usual multiplication on \( \mathbb{Q} \) and the sum \enquote{$+$} is the usual addition on \( \mathbb{Q} \), the neutral element of the product is $1$ and $0$ is the neutral element of summation (and null-element of the product).
We will use notions and results from formal language and automata theory (see, e.g.,~\cite{KS_HopcroftMotwaniUllman2007}) and computational complexity (see, e.g.,~\cite{KS_AroraBarak2009}).
With \( \varepsilon \) we indicate the empty word.
A decision problem \( P \) is a formal language over \( \{0,1\} \), and a complexity class is a set of decision problems.
We write $P_1{\leq_{p}}P_2$ if the decision problem $P_1$ can be reduced to $P_2$ in polynomial time.
The following relationship between complexity classes is known:
\begin{equation*}
    \textbf{NC}  \subseteq \textbf{PTIME} 
    \subseteq \textbf{PSPACE}
\end{equation*}
\textbf{NC} consists of all problems decidable in polylogarithmic time on a parallel computer with a polynomial number of processors, 
\textbf{PTIME}%
consists of all problems decidable in polynomial time by a deterministic%
Turing machine,
and \textbf{PSPACE} is the class of all problems that can be solved with polynomial space by a deterministic Turing machine. 

\section{Abstract Argumentation}
\label{sec:aaf}
An \emph{abstract argumentation framework} (\acr{AAF}) \cite{DBLP:journals/ai/Dung95} is a directed graph $F=(A,R)$ where $A$ is a (finite) set of \emph{arguments} and $R$ is a relation $R\subseteq A \times A$. An argument $a$ is said to \emph{attack} an argument $b$ if $(a,b) \in R $. For a set $E\subseteq A$ we define
\begin{align*}
    E^+_F & = \{a\in A\mid \exists b\in E, bR a\} \\
    E^-_F & = \{a\in A\mid \exists b\in E, aR b\}
\end{align*}
If $E$ is a singleton set, i.\,e., $E=\{a\}$ for some $a\in A$, then we also just write $a^+_F$ (resp.\ $a^-_F$) for $\{a\}^+_F$ (resp.\ $\{a\}^-_F$). If the AF is clear in the context, we will omit the index.

A \emph{ranking-based semantics} \cite{DBLP:conf/sum/AmgoudB13a} is defined as follows.
\begin{definition}
A \emph{ranking-based semantics} $\rho$ is a function which maps an \acr{AAF} $F=(A,R)$ to a preorder $\succeq_{F}^{\rho}$ on $A$.   
\end{definition}
Intuitively $a \succeq_{F}^{\rho} b$ means that $a$ is at least as strong as $b$ in $F$.
We define the usual abbreviations as follows;
$a \succ^{\rho}_{F} b$ denotes \emph{strictly stronger}, i.\,e. $a \succeq^{\rho}_{F} b$ and $b \not\succeq^{\rho}_{F} a$. 
Moreover, $a \simeq^{\rho}_{F} b$ denotes \emph{equally strong}, i.\,e. $a \succeq^{\rho}_{F} b$ and $b \succeq^{\rho}_{F} a$. 

Many different ranking-based semantics exist, see \cite{Bonzon:2023} for a survey, but in this paper, we will only consider the \emph{discussion-based semantics} form \citeauthor{DBLP:conf/sum/AmgoudB13a} (\citeyear{DBLP:conf/sum/AmgoudB13a}).
Very central to this semantics is the notion of length of walks, from graph theory, which we define next.
\begin{definition}
	Let $F=(A,R)$ be an \acr{AAF}, $x\in A$ and $i>0$.
	The set of walks of length $i$ in $F$ that end in $x$ is denoted by $\Paths{x}{F}{i}$ and defined by:
	\begin{align*}
		\Paths{x}{F}{1} & = \{ (y,x) \mid y\in x^-_F\}\\
		\Paths{x}{F}{i} & = \bigcup_{y\in x^-_F} \{(z_1,\ldots,z_{i-1},y,x)\mid\\
		&\qquad(z_1,\ldots,z_{i-1},y)\in \Paths{y}{F}{i-1}\} 
	\end{align*}
	for all $i>1$.
\end{definition}
 The intuition behind the discussion-based semantics is that an argument is stronger if it has less \emph{attack walks}, i.\,e., walks of odd length that end in the argument, and more \emph{defense walks}, i.\,e., walks of even length that end in the argument. This is formalised by \citeauthor{DBLP:conf/sum/AmgoudB13a} (\citeyear{DBLP:conf/sum/AmgoudB13a}) as follows.
\begin{definition}
    Let $F=(A,R)$ be an \acr{AAF}, $x\in A$ and $i>0$. The \emph{discussion count} $\Dis(x)$ is an infinite vector of natural numbers $\Dis^F(x)=\langle \Dis^F_1(x),\Dis^F_2(x),\ldots\rangle$  with
    \begin{align*}
        \Dis^F_i(x) & = \left\{\begin{array}{ll}
                -|\Paths{x}{F}{i}| & \text{if $i$ is odd}\\
                |\Paths{x}{F}{i}| & \text{if $i$ is even}
            \end{array}\right.
    \end{align*}
\end{definition}
Now an argument $a$ is stronger than an argument $b$ if the discussion count of $a$ is \emph{lexicographically larger} than the discussion count of $b$. That is, for two (possibly infinite) vectors $\vec{x}=\langle x_1,x_2,\ldots,\rangle$ and $\vec{y}=\langle y_1,y_2,\ldots,\rangle$ we have $\vec{x}\geq^\text{lex}\vec{y}$ if either $x_i=y_i$ for all $i>0$ or there is $k>0$ with $x_k>y_k$ and $x_j=y_j$ for all $j=1,\ldots,k-1$.
\begin{definition}
    Let $F=(A,R)$ be an \acr{AAF}. The \emph{dicussion-based} semantics $\dbs$ is defined via $\succeq^\text{dbs}_F=\dbs(F)$ with
    \begin{align*}
        a\, \succeq^\text{dbs}_F\, b \text{~iff~} \Dis^F(a)\geq^\text{lex}\Dis^F(b)
    \end{align*}
    for all $a,b\in A$.
\end{definition}
\begin{figure*}[t]
	\begin{minipage}[b]{0.32\textwidth}
		\centering
		\scalebox{1}{
			\begin{tikzpicture}
				
				\node (a) at (0,0) [circle, draw,minimum size= 0.65cm] {$a$};
				\node (b) at (-1,-1) [circle, draw,minimum size= 0.65cm] {$b$};
				\node (c) at (-1,-2) [circle, draw,minimum size= 0.65cm] {$c$};
				
				\node (d) at (0,-1) [circle, draw,minimum size= 0.65cm] {$d$};
				\node (e) at (1,-1) [circle, draw,minimum size= 0.65cm] {$e$};
				\node (f) at (1,-2) [circle, draw,minimum size= 0.65cm] {$f$};
				\node (g) at (1,-3) [circle, draw,minimum size= 0.65cm] {$g$};

				\path[->] (a) edge (b);
				\path[->] (a) edge (d);
				\path[->] (a) edge (e);
				\path[->] (b) edge (c);
				\path[->] (d) edge (f);
				\path[->] (e) edge (f);
				\path[->] (f) edge (g);

		\end{tikzpicture}}\\[0.5em]
	\begin{tabular}{l|ccc}
		\( x \) & \( \Dis_1(x) \) & \( \Dis_2(x) \) & \( \Dis_3(x) \) \\ \hline
		\( a \) &    \( 0 \)      &     \( 0 \)     &     \( 0 \)     \\
		\( b \) &    \( -1 \)     &     \( 0 \)     &     \( 0 \)     \\
		\( c \) &    \( -1 \)     &     \( 1 \)      &     \( 0 \)      \\
		\( d \) &    \( -1 \)      &     \( 0 \)      &     \( 0 \)      \\
		\( e \) &    \( -1 \)      &     \( 0 \)      &     \( 0 \)      \\
		\( f \) &    \( -2 \)      &     \( 2 \)      &     \( 0 \)      \\
		\( g \) &    \( -1 \)      &     \( 2 \)      &     \( -2 \)
	\end{tabular}
		\caption{\acr{AAF} $F_{\ref{afex:simple_ex}}$ from Example \ref{ex:simple_ex} and\\ the respective discussion count until \( 3 \).}
		\label{tikz:simple_ex}
	\end{minipage}
	\begin{minipage}[b]{0.32\textwidth}
		\centering
		\scalebox{1}{
			\begin{tikzpicture}
				
				\node (a) at (0,0) [circle, draw,minimum size= 0.65cm] {$a$};
				\node (b) at (0,-1) [circle, draw,minimum size= 0.65cm] {$b$};
				\node (c) at (0,1) [circle, draw,minimum size= 0.65cm] {$c$};
				
				\node (d) at (-1,-0.5) [circle, draw,minimum size= 0.65cm] {$d$};
				\node (e) at (1,-0.5) [circle, draw,minimum size= 0.65cm] {$e$};
				\node (f) at (-1,0.5) [circle, draw,minimum size= 0.65cm] {$f$};
				\node (g) at (1,0.5) [circle, draw,minimum size= 0.65cm] {$g$};

				\path[->] (a) edge (e);
				\path[->] (a) edge (d);
				\path[->] (a) edge (g);
				\path[->] (a) edge (f);
				
				\path[->] (b) edge (a);
				\path[->] (c) edge (a);
				
				\path[->] (d) edge (b);
				\path[->] (e) edge (b);
				\path[->] (f) edge (c);
				\path[->] (g) edge (c);
				
		\end{tikzpicture}}\\[0.5em]
	\begin{tabular}{l|ccc}
	\( x \) & \( \Dis_1(x) \) & \( \Dis_2(x) \) & \( \Dis_3(x) \) \\\hline
	\( a \)           & \( -2 \)  & \( 4 \)   & \( -4 \)   \\
	\( b \)           & \( -2 \)  & \( 2 \)   & \( -4 \)   \\
	\( c \)           & \( -2 \)  & \( 2 \)   & \( -4 \)   \\
	\( d \)           & \( -1 \)  & \( 2 \)   & \( -4 \)   \\
	\( e \)           & \( -1 \)  & \( 2 \)   & \( -4 \)   \\
	\( f \)           & \( -1 \)  & \( 2 \)   & \( -4 \)   \\
	\( g \)           & \( -1 \)  & \( 2 \)   & \( -4 \)
\end{tabular}
		\caption{\acr{AAF} $F_{\ref{afex:double_loop}}$ from Example \ref{ex:double_loop} and\\ the respective discussion count until \( 3 \).}
		\label{tikz:double_loop}
	\end{minipage}
	\begin{minipage}[b]{0.32\textwidth}
	    \centering
		\begin{tikzpicture}
			
			\node (a) at (0,0) [circle, draw,minimum size= 0.65cm] {$h$};
			\node (b) at (0,-1) [circle, draw,minimum size= 0.65cm] {$j$};
			\node (c) at (0,1) [circle, draw,minimum size= 0.65cm] {$k$};
			
			\node (d) at (-1,-0.5) [circle, draw,minimum size= 0.65cm] {$l$};
			\node (e) at (1,-0.5) [circle, draw,minimum size= 0.65cm] {$m$};
			\node (f) at (-1,0.5) [circle, draw,minimum size= 0.65cm] {$n$};
			\node (g) at (1,0.5) [circle, draw,minimum size= 0.65cm] {$o$};
			
			\coordinate (im1) at ([xshift=-0.15cm,yshift=-0.5cm]f.west);
			\coordinate (im2) at ([xshift=-0.15cm,yshift=0.25cm]f.west);
			\coordinate (im3) at ([xshift=0.1cm,yshift=0.5cm]f.west);
			
			\path[->] (a) edge (e);
			\path[->] (a) edge (d);
			\path[->] (a) edge (g);
			\path[->] (a) edge (f);
			
			\path[->] (b) edge (a);
			\path[->] (c) edge (a);
			
			\draw[-] (d) -- (im1) -- (im2) -- (im3) edge[->] (c);
			\path[->] (e) edge (b);
			\path[->] (f) edge (c);
			\path[->] (g) edge (c);
			
		\end{tikzpicture}\\[0.5em]
		\begin{tabular}{l|ccc}
			\( x \) & \( \Dis_1(x) \) & \( \Dis_2(x) \) & \( \Dis_3(x) \) \\\hline
			\( h \)           & \( -2 \)  & \( 4 \)   & \( -4 \)   \\
			\( j \)           & \( -1 \)  & \( 1 \)   & \( -2 \)   \\
			\( k \)           & \( -3 \)  & \( 3 \)   & \( -6 \)   \\
			\( l \)           & \( -1 \)  & \( 2 \)   & \( -4 \)   \\
			\( m \)           & \( -1 \)  & \( 2 \)   & \( -4 \)   \\
			\( n \)           & \( -1 \)  & \( 2 \)   & \( -4 \)   \\
			\( o \)           & \( -1 \)  & \( 2 \)   & \( -4 \)
		\end{tabular}
		\caption{\acr{AAF} $F_{\ref{afex:nonismophic}}$ from Example \ref{ex:nonismophic} and\\ the respective discussion count until \( 3 \).}
		\label{tikz:nonismophic}
	\end{minipage}
\end{figure*}
\begin{example}\label{ex:simple_ex}
Let \refstepcounter{afexample} $F_{\theafexample} = (A_{\theafexample},R_{\theafexample})$ \label{afex:simple_ex} with 
\begin{align*}
	A_{\ref{afex:simple_ex}} &= \{a,b,c,d,e,f,g\} \\
	R_{\ref{afex:simple_ex}} & = \{(a,b),(a,d),(a,e), (b,c), (d,f),(e,f),(f,g)\}
\end{align*}
be the \acr{AAF} as depicted in Figure \ref{tikz:simple_ex}.  
Argument $a$ is unattacked, so it is the strongest argument. The arguments $b$, $d$ and $e$ have the same attackers (argument $a$) and are therefore equally strong. The arguments $c$ and $g$ are both attacked by one argument each, however $g$ is defended by two arguments, while $c$ is only defended by one argument, thus $g \succ^{\text{dbs}}_{F_{\ref{afex:simple_ex}}} c$. Argument $f$ is attacked by two arguments, rendering \( f \) the weakest one according to the discussion-based semantics. The ranking is:
\begin{equation*}
	a  \succ^{\text{dbs}}_{F_{\ref{afex:simple_ex}}} g  \succ^{\text{dbs}}_{F_{\ref{afex:simple_ex}}}c  \succ^{\text{dbs}}_{F_{\ref{afex:simple_ex}}} e  \simeq^{\text{dbs}}_{F_{\ref{afex:simple_ex}}} d  \simeq^{\text{dbs}}_{F_{\ref{afex:simple_ex}}} b  \succ^{\text{dbs}}_{F_{\ref{afex:simple_ex}}} f
    \tag*{\KSendexample}
\end{equation*}    
\end{example}
In the remainder of this paper, we are interested in the following decision problems:

\begin{center}\begin{minipage}{\columnwidth}
        \noindent\problemDisEquivalence\\
\textbf{Input}: \acr{AAF} $F=(A,R)$ and \( a,b\in A \).\\
\textbf{Question:} Does $a\, \simeq^\text{dbs}_F\, b$ hold?
    \end{minipage}\\[0.5em]
\begin{minipage}{\columnwidth}
        \noindent\problemDisStronger\\
\textbf{Input}: \acr{AAF} $F=(A,R)$ and \( a,b\in A \).\\
\textbf{Question:} Does $a\, \succeq^\text{dbs}_F\, b$ hold?
    \end{minipage}
\end{center}

Note that the above questions can be easily answered if $F$ is \emph{acyclic}, since the number of all walks ending in arguments $a$ and $b$ are finite and their numbers can easily be compared in polynomial time. However, if $F$ contains cycles, the set of walks ending in $a$ and $b$, respectively, becomes infinite, with possibly a complex underlying structure.
\begin{example}\label{ex:double_loop}
Let  \refstepcounter{afexample} $F_{\theafexample} = (A_{\ref{afex:double_loop}},R_{\ref{afex:double_loop}})$\label{afex:double_loop} with 
\begin{align*}
    A_{\ref{afex:double_loop}} &= \{a,b,c,d,e,f,g\} \\
    R_{\ref{afex:double_loop}} & = \{ (a,d), (a,e), (a,f), (a,g), (b,a),\\ &\hspace{0.75cm} (c,a), (d,b), (e,b), (g,c), (f,c)\}
\end{align*}
as depicted in Figure \ref{tikz:double_loop}. 
Arguments $d,g,f$ and $g$ have the same attacker $a$, while the other three arguments have two attackers each. For argument $a$, there are four walks of length 2 and length 3 ending in that argument. 
For length 4, there are already eight walks of length 4 ending in argument $a$.
These numbers increase significantly for longer lengths.     
\hfill{\KSendexample}
\end{example}
One might guess that $a\, \simeq^\text{dbs}_F\, b$ only holds if $F$ is somewhat isomorphic about $a$ and $b$, i.\,e., the topological structure of $F$ is the same for both $a$ and $b$. This is not the case as the following example shows.
\begin{example}\label{ex:nonismophic}
Consider \refstepcounter{afexample} $F_{\theafexample} = (A_{\ref{afex:nonismophic}},R_{\ref{afex:nonismophic}})$\label{afex:nonismophic} with
\begin{align*}
	A_{\ref{afex:nonismophic}}& = \{h,i,j,k,l,m,n,o\}\\
	R_{\ref{afex:nonismophic}} & = \{ (h,l), (h,m), (h,n), (h,o), (j,h),\\ & \hspace{0.75cm}(k,h), (l,k), (m,j), (n,k), (o,k)\} \ ,
\end{align*}
which is also presented in Figure \ref{tikz:nonismophic}, we compare $F_{\ref{afex:nonismophic}}$ with  $F_{\ref{afex:double_loop}}$ from Example~\ref{ex:double_loop}.
Note that there are two walks of length $2$ to $a$ over $b$ and two walks of length $2$ to $a$ over $c$, but there is only one walk of length $2$ to $h$ over $j$ and three walks of length $2$ to $h$ over $k$.
Thus, the topological structure of \( F_{\theafexample} \) and \( F_{\ref{afex:double_loop}} \) are different, but, nonetheless, as we show at the end of this paper, ${a\, \simeq^\text{dbs}_{F_{\text{joint}}} h}$ holds for the \acr{AAF} \( F_{\text{joint}} = (A_{\ref{afex:double_loop}} \cup  A_{\ref{afex:nonismophic}},R_{\ref{afex:double_loop}} \cup R_{\ref{afex:nonismophic}}) \).
\hfill{\KSendexample}
\end{example}
The above examples may raise the question of whether there is a bound at all for a walk length $i$ for which one has to consider $\Dis^F_i(a)$ and $\Dis^F_i(b)$ in order to be sure that $a\, \simeq^\text{dbs}_F\, b$ holds.
One might argue that it should suffice to only consider walks up to length $|A|+1$ since all longer walks contain cycles and are therefore ``combinations'' of walks already seen. However, there is no apparent formal reason for such an assumption, since it remains difficult (or even indeterminate), from a combinatorial point of view, to find the exact number of walks of a certain length, given that we already know the number of walks of smaller lengths. In fact, no closed-form solution is known for deriving the number of walks of length $k$ for any $k$, given that we know the lengths of walks of length $1,\ldots,k-1$~\cite{Flajolet_Wordcounting}.

\section{Graph-Theoretical Perspective}
\label{sec:graphperspective}
We now take a more general perspective than for abstract argumentation frameworks
and consider general directed graphs and graph theory notation, where we write $G=(V,E)$ for a directed graph. As for argumentation frameworks, we denote by $\Paths{\vertexa}{G}{i} $ the walks of length $i$ in a graph $G=(V,E)$ that end in $\vertexa\in V$. If $G$ is clear from the context, we may omit \( G \).

\noindent We say that two vertices \( \vertexa,\vertexb \) of \( G \) \emph{agree on the numbers of walks} if for each length the number of walks of that length ending in these vertices is the same, i.e., \( |\Paths{\vertexa}{}{i}| = |\Paths{\vertexb}{}{i}| \) holds for all \( i \in \mathbb{N} \).
Next, we define the corresponding decision problem as follows.
\begin{center}\begin{minipage}{\columnwidth}
\noindent{\problemGraphPathLength}\\
\textbf{Input}: Graph \( G=(V,E) \) and two vertices \( \vertexa,\vertexb\in V \).\\
\textbf{Question:} Do \( \vertexa \) and \( \vertexb \) agree on the numbers of walks?
    \end{minipage}
\end{center}

There is an immediate relationship between the problem {\problemGraphPathLength} and the problem \( \problemDisEquivalence \). 
\pagebreak[3]
\begin{theorem}
    \label{thm:DisEqToGraphWalks}
	\( \problemDisEquivalence \,{\leq_{p}}\, {\problemGraphPathLength} \)
\end{theorem}
\begin{proof}
	The reduction is the identity function
	\begin{equation*}
		(F,a,b) \mapsto (F,a,b)
	\end{equation*}
	which is easily computable in polynomial time.
	We have that \( { a\, \simeq^\text{dbs}_F\, b } \) holds exactly whenever \( \Dis^F_n(a) = \Dis^F_n(b) \) for all \( n\in\mathbb{N} \). The latter is the case when \( |\Paths{a}{}{n}| = |\Paths{b}{}{n}| \) hold for all \( n \in \mathbb{N} \).
\end{proof}

We establish connections between \( |\Paths{v}{}{i}| \) and other notions. 
With \( N(\vertexa) = \{ \vertexb \in V \mid (\vertexb,\vertexa) \in E \}  \) we denote the set of in-neighbours of \( \vertexa \), i.\,e., those vertices that have an edge going out that ends in \( \vertexa \). 
We define the notion
\begin{align*}
    T_{i} & :  V \to \mathbb{N}_0\\
    T_{1}(\vertexa) & = |N(\vertexa)|\\
    T_{i+1}(\vertexa) & = \sum_{\vertexb\in N(\vertexa)} T_{i}(\vertexb) \ \ \ ,
\end{align*}
for which one can show that \(  T_{i}(v)\) coincides with the number of walks of length \( i \) ending in~\( \vertexa \).

\begin{proposition}
    \label{prop:walkTn}
    For all \(\vertexa\in V\) and \( u\in\mathbb{N}\) holds:    
    \begin{equation*}
        |\Paths{\vertexa}{}{i}| = T_i(\vertexa)
    \end{equation*}
\end{proposition}
\begin{proof}
	For \( i=1 \), this is immediate by the definitions.
	For \( i\geq 1 \), we have that 
	\begin{align*}
		& |\Paths{\vertexa}{}{i+1}|\\
		& = \bigg| \bigcup_{y\in N(\vertexa)} \{(z_1,...,z_{i},y,\vertexa)\mid(z_1,...,z_{i},y)\in \Paths{y}{}{i}\}  \bigg|
	\end{align*}
	Because every tuple \( (z_1,...,z_{i},y) \) ends with the unique element \( y \), the union is disjoint.
	Hence, we obtain the following:
	\begin{align*}
		& |\Paths{\vertexa}{}{i+1}|\\
		& =  \sum_{y\in N(\vertexa)} \bigg|\{(z_1,...,z_{i},y,\vertexa)\mid(z_1,...,z_{i},y)\in \Paths{y}{}{i}\}  \bigg|\\
		& =  \sum_{y\in N(\vertexa)} T_{i}(\vertexa)  = T_{i+1}(\vertexa) \qedhere
	\end{align*}
\end{proof}

Another representation of \( |\Paths{v}{}{i}| \) is via the adjacency matrix.
The adjacency matrix \( M \in \{0,1\}^{|V|\times|V|}  \) of \( G\) is the matrix with \( M[\vertexa,\vertexb] = 1 \) if and only if \( (\vertexa,\vertexb)\in E \); whereby, \( M[\vertexa,\vertexb] \) stands for \( M[i,j] \) presuming that $V=\{ v_1,\ldots,v_m\}$ and \( v=v_i\) and \( v' = v_j \).
We write $M[a]$ as an abbreviation for $ \sum_{b \in A} M[b,a] $.
It is well-known, see, e.g., \citeauthor{KS_TinhoferAlbrechtMayrNoltemeierSyslo1990},~\citeyear{KS_TinhoferAlbrechtMayrNoltemeierSyslo1990}, that $|\Paths{\vertexa}{}{i}|$ can be represented as the sum of the column of \( v \) in the $i$-th power of the adjacency matrix $M$.
\begin{proposition}
\label{prop:walks_matrixpotency}
    For all \( v\in V \) and \( i \in \mathbb{N} \) holds:
	\begin{equation*}		
|\Paths{\vertexa}{}{i}| = M^i[v]
	\end{equation*}
\end{proposition}
\begin{proof}
	For \( i=1 \), we have
	\begin{align*}
		\sum_{j=1}^{m}A[v_j,\vertexa] &= \sum_{\vertexb\in V, (\vertexb,\vertexa)\in E}1=|\Paths{\vertexa}{}{1}| 
	\end{align*}
	For \(i=n+1\) we have
	\begin{align*}
		\sum_{j=1}^{m}A^{i}[v_j,\vertexa]&=\sum_{j=1}^{m}(A^{i-1}\cdot A)[v_j,\vertexa]\\&=  \sum_{\vertexb\in V, (\vertexb,\vertexa)\in E}\sum_{j=1}^{m} A^{i-1}[v_j,\vertexb]\\&=\sum_{\vertexb\in V, (\vertexb,\vertexa)\in E}|\Paths{\vertexb}{}{i-1}|\\&=|\Paths{\vertexa}{}{i}|
		\qedhere
	\end{align*}    
\end{proof}

\section{$\mathbb{Q}$-Automata and their Equivalence}
\label{sec:semiringautomatEQ}

We provide background on semiring-automata and the complexity of computing whether two respective automata over the semiring of rational numbers are equivalent.

\subsection{$\mathbb{Q}$-Automata}
\label{sec:Qautomata}
A generalization of finite (nondeterministic) automata is its extension to arbitrary semirings and its connection to formal power series~\cite{KS_SalomaaSoittola1978,KS_BerstelReutenauer1988}.
In the following, we introduce automata over the semiring of rational numbers~\( \mathbb{Q} \).
For a finite alphabet $\Sigma$, a formal power series (over the semiring $\mathbb{Q}$ and \( \Sigma^* \)) is a function $ S : \Sigma^* \to \mathbb{Q}$.

Semiring-automata follow the same intuition as nondeterministic finite automata~\cite{KS_HopcroftMotwaniUllman2007,KS_RabinScott1959}, but assign weights from an underlying semiring to the initial states, to the transitions and to the final states.
The semantics is then that every run yields a weight, which is the product of the weights along the run.

\pagebreak[3]
\begin{definition}
A \( \mathbb{Q} \)-automaton (over the finite alphabet $\Sigma$) is a tuple \( \mathcal{A}=\tuple{Q,\Sigma,\delta,I,F} \) such that
    \begin{itemize}
        \item \( Q \) a finite set (the states),
        \item \( \Sigma \) a finite alphabet,
        \item \( \delta :  Q \times \Sigma \times Q \to \mathbb{Q} \) the transition function,
        \item \( I : Q \to \mathbb{Q} \) initial values, and
        \item \( F : Q \to \mathbb{Q} \) final values.
    \end{itemize}
\end{definition}
For all states \( q,p \in Q \), all \( n>1 \) and all \( a,a_1,\ldots,a_n \in \Sigma \) define the extended transition function $\delta^*: Q\times\Sigma^{*}\times Q \to \mathbb{Q}$:
\begin{align*}
    \delta^{*}(q,\varepsilon,p) & = \begin{cases}
        1 & \text{, if } p = q\\
        0 & \text{, if } p \neq q
    \end{cases}\\
    \delta^{*}(q,a,p) & = \delta(q,a,p) \\
    \delta^{*}(q,a_1\ldots a_n,p) & = \sum_{r \in Q} \delta(q,a_1\ldots a_{n-1},r) \cdot \delta(r,a_n,p)  
\end{align*}
For each state \( q\in Q \) and each word \( w\in \Sigma^* \) define:
\begin{align*}
\mathcal{A}(q,w) & = \sum_{p\in Q} I(p) \cdot \delta^{*}(p,w,q) \cdot F(q) 
\end{align*}
That is, \( \mathcal{A}(q,w) \) is the sum over all runs for the word \( w \) that end in \( q \), where weights are multiplied.
Each automaton $\mathcal{A}$ induces a formal power series $\llbracket\mathcal{A}\rrbracket$ which is defined by:
\begin{align*}
    &\llbracket\mathcal{A}\rrbracket : \Sigma^* \to \mathbb{Q} &&
    \llbracket\mathcal{A}\rrbracket (w) = \sum_{q\in Q} \mathcal{A}(q,w)
\end{align*}
Note that, when one replaces \( \mathbb{Q} \) by the semiring of Booleans \( \mathbb{B}=(\{0,1\},\land,\lor,0,1) \), one obtains nondeterministic finite automata. For instance, in \( \mathbb{B}\)-automata \( \mathcal{A}(q,w)=1 \) holds when there is an accepting run for the word \( w \) that starts in \( q \).

\subsection{Equivalence of $\mathbb{Q}$-Automata}
\label{sec:QautomataEQ}
We say that two \( \mathbb{Q} \)-automata are equivalent if they induce the same formal power series.
The following problem of equivalence has been considered in automata theory.

\begin{center}\begin{minipage}{\columnwidth}
\noindent{\problemAutomataEquivalence}\\
		\textbf{Input}: $\mathbb{Q}$-automata \( \mathcal{A}_{1} \) and \( \mathcal{A}_{2} \) over the same alphabet \( \Sigma \).\\
		\textbf{Question:} Does \( \llbracket\mathcal{A}_{1}\rrbracket = \llbracket\mathcal{A}_{2}\rrbracket \) hold?
	\end{minipage}
\end{center}

The general methodology for deciding the equivalence of semiring automata is attributed to \KSciteauthoryear{KS_Schuetzenberger1961}. Decidability for the semiring $\mathbb{Q}$ is given by \KSciteauthoryear{KS_EsikMaletti2010}.
The detailed analysis of the computational complexity of {\problemAutomataEquivalence} is given by \KSciteauthoryear{KS_Tzeng1996}. 
\begin{proposition}[\citeauthor{KS_Tzeng1996},\citeyear{KS_Tzeng1996}]
	\label{prop:complxityAutomataStuff}
	{\problemAutomataEquivalence} is in \textbf{NC}.
\end{proposition}

We will also make use of the following result by \KSciteauthoryear{KS_Kiefer2020} that shows when solving {\problemAutomataEquivalence} in \textbf{PTIME} (instead of \textbf{NC}), one obtains a witness whenever the input is non-equivalent.
\begin{proposition}[{\citeauthor{KS_Kiefer2020}, \citeyear{KS_Kiefer2020}}]
    \label{prop:complxityAutomataStuffWitness}
    There is an algorithm that decides \textsc{$\mathbb{Q}$-Equivalence} in polynomial time for two given automata \( \mathcal{A}_{1}=\tuple{Q_{1},\Sigma,\delta_{1},I_{1},F_{1}} \) and \( \mathcal{A}_{2}=\tuple{Q_{2},\Sigma,\delta_{2},I_{2},F_{2}} \), and if \( \mathcal{A}_{1} \) and \( \mathcal{A}_{2} \) are not equivalent, then a word $w \in \Sigma ^*$ with $\llbracket\mathcal{A}_{1}\rrbracket(w) \neq \llbracket\mathcal{A}_{2}\rrbracket(w) $ of maximal size $|Q_1|{+}|Q_2|{-}1$ is provided.
\end{proposition}
The proof of the above proposition is constructive and provides an algorithm. We will outline this algorithm now, but for a more deliberated and detailed presentation, we refer to \KSciteauthoryear{KS_Kiefer2020}.
For the algorithm, one computes the \emph{linerar representation} $\tuple{\alpha,M,\eta}$ of an $\mathbb{Q}$-automaton \( \mathcal{A}=\tuple{Q,\Sigma,\delta,I,F} \) with $Q=\{ q_0, \ldots, q_{n-1}\} $, which is that
\begin{itemize}
    \item $\alpha \in \mathbb{Q}^{1\times n}$ is the initial vector (a row-vector over $\mathbb{Q}$), such that $\alpha[0,i]=I(q_i)$ for all $i\in\{0,\ldots,n\}$,
    \item $ M: \Sigma \to \mathbb{Q}^{n\times n} $ assigns to every sign a quadratic matrix (the transition matrix), such that $M(a)[i,j]=\delta(q_i,a,q_j)$ for all $i\in\{0,\ldots,n\}$, and
    \item $\eta \in \mathbb{Q}^{n\times 1}$  is the final vector, a column-vector over $\mathbb{Q}$ such that $\alpha[i,0]=F(q_i)$ for all $i\in\{0,\ldots,n\}$.
\end{itemize}
One can obtain the power series $\llbracket\mathcal{A}\rrbracket$ by means of the matrix multiplication, i.\,e., for all words $w=a_0\ldots a_m \in \Sigma^*$ it holds 
\begin{align*}
    \llbracket\mathcal{A}\rrbracket (w) =\alpha \cdot M(w) \cdot \eta
\end{align*}
whereby $ M(w) = M(a_0)\cdot \ldots \cdot M(a_m) $.
Because $\mathbb{Q}$ is also a field, $\mathbb{Q}^n$ is also a vector space over $\mathbb{Q}$ of dimension $\dim(\mathbb{Q}^n)=n$. For a set of vectors $S \subseteq \mathbb{Q}^n$, we denote by $\mathrm{span}(S) $ the smallest subset of $\mathbb{Q}^n$ that contains $S$ and is itself a vector space (with the same scalars and operators from $\mathbb{Q}^n$). 
The \emph{forward-space} $\forwardspace$ of a $\mathbb{Q}$-automaton \( \mathcal{A} \) is defined by
\begin{equation*}
    \forwardspace = \mathrm{span}(\{ \alpha \cdot M(w) \mid w \in \Sigma^* \}) \ .
\end{equation*}
Because $\forwardspace$ is a subspace of \( \mathbb{Q}^n \), the dimension of the vector space $\forwardspace$ is bound by  $ \dim(\forwardspace) \leq n $.
Hence, $\forwardspace$ has a base \( \baseofA \) consisting of $\dim(\forwardspace) $ many vectors.
Notably, for each base $b \in \baseofA$, there is a word $w_b$ of length $|w_b| \leq \dim(\forwardspace)-1$ such that $b=\alpha \cdot M(w_b)$.
One computes a base $\baseofA$ of $\forwardspace$ by initially set \( \baseofA = \{ \alpha \} \) and \( W = \Sigma \). In each iteration dequeue $w_b\in W$ and if \( b =\alpha \cdot M(w_b) \) is linear independent to \( \baseofA \), add \( b \) to \( \baseofA \) and set $ W = W \cup \{ w_ba \mid a \in \Sigma \} $. All operations are feasible in polynomial time; linear independence can be tested by computing the reduced row echelon form~\cite{KS_Edmonds1967}.
Given $\baseofA$, we decide in polynomial time the emptiness problem for $\mathcal{A}$, i.\,e., decide if $\llbracket\mathcal{A}\rrbracket(w)=0$ holds for all $w\in\Sigma^*$.
This can be done by checking for each vector $b\in\baseofA$ if $b\cdot\eta =0$ holds, i.\,e., all bases are orthogonal to the final vector $\eta$.
Because matrix multiplication can be computed in polynomial time and because $\baseofA$ has at most $n$ elements, checking for emptiness is also in polynomial time.

For computing {\problemAutomataEquivalence} in \textbf{PTIME}, one employs the method of checking emptiness from above. First, construct (in polynomial time) from the input $\mathcal{A}_1$ and $\mathcal{A}_2$ an automaton $\mathcal{A}_{\text{\textminus}}$ which runs $\mathcal{A}_1$ and $\mathcal{A}_2$ in parallel and negates the initial values of $\mathcal{A}_2$.
When \( \tuple{\alpha_i,M_i,\eta_i} \) is the linear representation of $\mathcal{A}_i$, the linear representation \( \tuple{\alpha_{\text{\textminus}},M_{\text{\textminus}},\eta_{\text{\textminus}}} \) of \( \mathcal{A}_{\text{\textminus}} \) is
\begin{align*}
    \alpha_{\text{\textminus}} & = \begin{bNiceArray}{cc} \alpha_1 & -\alpha_2 \end{bNiceArray} &
    M_{\text{\textminus}}(a) & = \begin{bNiceArray}{cc}%
        M_1(a) & 0 \\
         0 & M_2(a)
    \end{bNiceArray} &
    \eta_{\text{\textminus}} & = \begin{bNiceArray}{c}
        \eta_1 \\
        \eta_2
    \end{bNiceArray}
\end{align*}
where \( \alpha_{\text{\textminus}} \in \mathbb{Q}^{1\times(|Q_1|+|Q_2|)} \), \( \eta_{\text{\textminus}} \in \mathbb{Q}^{(|Q_1|+|Q_2|)\times 1} \) and \( M_{\text{\textminus}}(a) \in \mathbb{Q}^{(|Q_1|+|Q_2|)\times(|Q_1|+|Q_2|)} \) for each \( a\in\Sigma \).
 The automaton $\mathcal{A}_{\text{\textminus}}$ has $|Q_1|+|Q_2|$ states. By construction, we have $\llbracket\mathcal{A}_{1}\rrbracket = \llbracket\mathcal{A}_{2}\rrbracket$ exactly when \( \mathcal{A}_{\text{\textminus}} \) is empty, i.e., $\llbracket\mathcal{A}_{\text{\textminus}}\rrbracket(w)=0$ holds for all $w\in\Sigma^*$.
We check emptiness of \( \mathcal{A}_{\text{\textminus}} \) as described above, i.e., by computing the base $\mathcal{B}_{{\text{\textminus}}}$ of the forward-space \( \mathcal{F}_{\!\!\mathcal{A}_{\text{\textminus}}} \) and check then orthogonality of $\mathcal{B}_{{\text{\textminus}}}$ with \( \eta_{\text{\textminus}} \). When $\mathcal{A}_{\text{\textminus}}$ is non-empty, above procedure guarantees a word $w_b \in \Sigma^{*}$ of length $|w_b| \leq |Q_1|+|Q_2|-1$ that witnesses $\llbracket\mathcal{A}_{1}\rrbracket \neq \llbracket\mathcal{A}_{2}\rrbracket$.

Note that the \textbf{NC} complexity of {\problemAutomataEquivalence} is a rather surprising phenomenon, as similar problems are usually much harder, e.\,g., \textbf{PSPACE}~\cite{KS_StockmeyerMeyer1973,KS_Tzeng1996} or even undecidable~\cite{KS_Sakarovitch2009}.

\section{Complexity of {\normalfont\problemGraphPathLength}}
\label{sec:Graph2Semiring}
In this section, we will show how the problem of deciding whether two vertices \( \vertexa,\vertexb \) \emph{agree on the numbers of walks}, i.e., if for each length the number of walks of that length that end in each vertex is the same (\problemGraphPathLength), can be reduced to the equivalence of semiring-automata over \( \mathbb{Q} \) ({\problemAutomataEquivalence}), yielding a \textbf{PTIME} complexity.
 
The central insight of this section is that one can construct for a vertex \( \vertexa \in V \) of a graph \( G=(V,E) \) an \( \mathbb{Q} \)-automaton \( \mathcal{A}_{G}^\vertexa \) over a unary alphabet, e.g., \( \Sigma=\{s\} \), such that the formal series \( 
\llbracket\mathcal{A}_{G}^\vertexa\rrbracket \) yields the numbers of walks of length \( i \) that end in \( \vertexa \) for the word \( s^i \).
Formally, for a graph \( G=(V,E) \) and a vertex \( \vertexa \in V \), we define the following $\mathbb{Q}$-automaton:
\begin{align*}
    \mathcal{A}_{G}^\vertexa & = (V,\Sigma,\delta,I,F) \\
    \Sigma & = \{ s \} &  
    I(q) & = 1 \\
    \delta(q,a,p) & = \begin{cases}
        1 & \text{, if } (q,p) \in E \\
        0 & \text{, if } (q,p) \notin E
    \end{cases} &
    F(q) & = \begin{cases}
        1 & q = \vertexa \\
        0 & q \neq \vertexa
    \end{cases} 
\end{align*}
The states of the automaton \( \mathcal{A}_{G}^{\vertexa} \) are the vertices $V$, the alphabet \( \Sigma= \{ s \} \) is unary, and $v$ is the only state to which we assign a non-zero weight. The transition function is (quintessentially) the edge relation of the graph \( G \), i.\,e., the transition from \( \vertexa \) to \( \vertexb  \) when reading \( a \) has weight \( 1 \) when there is an edge from \( \vertexa \) to \( \vertexb \) in \( G\); otherwise the weight is \( 0 \).

\pagebreak[3]
\begin{lemma}
\label{lem:eqgraphautomata}
    For each Graph \( G=(V,E) \), each vertex \( v\in V \) and each \( i\in\mathbb{N} \) holds \( T_i(\vertexa)=\mathcal{A}_{G}^{\vertexa}(\vertexa,s^i) \).
\end{lemma}
\begin{proof}
	We start by showing for each $q \in V$ and $i\in\mathbb{N}$ holds:
	\begin{equation*}
		T_i(q)= \sum_{p\in V} \delta^{*}(p,s^i,q)
	\end{equation*}
		The proof is by induction.
		For these case of \( i=1 \), we have
		\begin{align*}
			T_1(q) & = |N(q)| = |\{ p\in V \mid (p,q) \in E \}\\
			& = \sum_{p\in V} \left(\begin{cases}
				1 & \text{, if } (p,q) \in E \\
				0 & \text{, if } (p,q) \notin E
			\end{cases}\right) \\
			& = \sum_{p\in V} \delta(p,s,q) = \sum_{p\in V} \delta^{*}(p,s,q)
		\end{align*}
		For the case of \( i>1 \), we have
		\begin{align*}
			& \sum_{p\in V} \delta^{*}(p,s^i,q) \\
			& = \sum_{p\in V} \sum_{r \in V} \delta^*(p,s^{i-1},r) \cdot \delta(r,a,q)\\
			& = \sum_{p\in V} \sum_{r \in V} \left( \begin{cases}
				\delta^*(p,s^{i-1},r) &, (r,q) \in E\\
				0 &, (r,q) \notin E
			\end{cases} \right)\\
			& = \sum_{p\in V} \sum_{r \in N(q)} 
			\delta^*(p,s^{i-1},r) \\
			& = \sum_{r \in N(q)}\sum_{p\in V}  
			\delta^*(p,s^{i-1},r) \\
			& = \sum_{r \in N(q)} T_{i-1}(r) \tag{induction hypothesis} \\
			& = T_{i}(q) 
		\end{align*}
	Finally, we show \( T_i(\vertexa)=\mathcal{A}_{G}^{\vertexa}(\vertexa,s^i) \):
	\begin{align*}
		\mathcal{A}_{G}^{\vertexa}(\vertexa,s^i) & = \sum_{p \in V} I(p) \cdot  \delta^*(p,s^i,\vertexa) \cdot F(\vertexa)\\
		& = \sum_{p \in V} 1 \cdot  \delta^*(p,s^i,\vertexa) \cdot 1 = \sum_{p \in V} \delta^*(p,s^i,\vertexa) \\
		& = \sum_{p \in V} \delta^*(p,s^i,\vertexa) = T_i(\vertexa)
	\end{align*}
	This completes the proof of Lemma~\ref{lem:eqgraphautomata}.
\end{proof}

We obtain the following theorem by combining Proposition~\ref{prop:complxityAutomataStuff} and Lemma~\ref{lem:eqgraphautomata}. 
\begin{theorem}
    \label{thm:maingraph}
	The following holds:
	\begin{enumerate}[(a)]
		\item 
        {\problemGraphPathLength} \( \leq_p \) {\problemAutomataEquivalence}
		\item 
		{\problemGraphPathLength} is in \textbf{PTIME}.
	\end{enumerate}
\end{theorem}
\begin{proof}
	Clearly, (b) follows from (a) and Proposition~\ref{prop:complxityAutomataStuff}.
	For (b), we consider 
    the construction of \( \mathcal{A}_{G}^v \). One constructs \( \mathcal{A}_{G}^v \) easily in polynomial time in the size of the graph \( G=(V,E) \).
	Lemma~\ref{lem:eqgraphautomata} guarantees the following property:
	\begin{multline*}
	\Big( T_i(\vertexa) = T_i(\vertexb) \text{ for all \( i \,{\in}\, \mathbb{N} \)} \Big) \\
    \Leftrightarrow \Big( \mathcal{A}_{G}^{v_{\!1}}(\vertexa,s^i) =  \mathcal{A}_{G}^{v_{\!2}}(\vertexb,s^i) \text{ for all \( i\in \mathbb{N} \)} \Big)
	\end{multline*}
Now note that the following property holds:
	\begin{equation*} 
        \Big( \mathcal{A}_{G}^{\vertexa}(\vertexa,s^i) \,{=}\,  \mathcal{A}_{G}^{\vertexb}(\vertexb,s^i) \text{ for all \( i \,{\in}\, \mathbb{N} \)} \Big)
    \Leftrightarrow
   \llbracket \mathcal{A}_{G}^{\vertexa} \rrbracket \,{=}\,
   \llbracket \mathcal{A}_{G}^{\vertexb} \rrbracket
	\end{equation*}
	This is because \( u \) is in \( \mathcal{A}_{G}^{u} \) the only state non-zero final value for \( u \in \{\vertexa,\vertexb\} \).
    Moreover, because \( \mathcal{A}_{G}^{\vertexa} \) and \( \mathcal{A}_{G}^{\vertexb} \) operate on the unary alphabet \( \Sigma=\{s\} \), for each $i\in\mathbb{N}_0$ the word \( s^i \) is the one and only word in \( \Sigma^i \).
    Due to the above insights,
    \begin{equation*}
    (G,\vertexa,\vertexb)	\mapsto (\mathcal{A}_{G}^{\vertexa},\mathcal{A}_{G}^{\vertexb})
    \end{equation*}
    is a polynomial-time many-one reduction from  {\problemGraphPathLength} to {\problemAutomataEquivalence}.
\end{proof}

By combining Theorem~\ref{thm:maingraph}, respectively Lemma~\ref{lem:eqgraphautomata}, with Proposition~\ref{prop:complxityAutomataStuffWitness}, we obtain the following central insight.
\begin{theorem}
    \label{thm:limitgraph}
    For every graph \( G=(V,E) \) and all vertices \( \vertexa,\vertexb \in V \) that do not agree on the number of walks, it holds
    \begin{equation*}
        \min\{\  i \in \mathbb{N} \mid \  |\Paths{\vertexa}{}{i}| \neq |\Paths{\vertexb}{}{i}| \ \} \leq 2 |V| - 1 \ .
    \end{equation*}
\end{theorem}
Intuitively, Theorem~\ref{thm:limitgraph} states that the size of the smallest witness for non-agreement on the number of walks is bound linear in the amount of vertices.
\section{Complexity of Discussion-Based Semantics}
\label{sec:ComplexityDBS}
In this section, we settle the complexity of the discussion-based semantics.
Specifically, we show that {\normalfont{\problemDisEquivalence}} and {\normalfont{\problemDisStronger}} are decidable in polynomial time.
The preceded sections provide the following sequence of reductions:
\begin{equation*}
    \problemDisEquivalence \leq_{p} \problemGraphPathLength \leq_{p} \problemAutomataEquivalence
\end{equation*}
By employing Proposition~\ref{prop:complxityAutomataStuff}, we obtain from Theorem~\ref{thm:DisEqToGraphWalks} and Theorem~\ref{thm:maingraph} decidability within polynomial time.
\begin{theorem}
    {\problemDisEquivalence} is in \textbf{PTIME}.
\end{theorem}
In the following example, we carry out the line of reductions and show how the final result is computed.
\begin{example}
\label{ex:fullexample}
Consider the \acr{AAF}
	\begin{equation*}
		F_{\text{joint}} = (A_{\ref{afex:double_loop}} \cup  A_{\ref{afex:nonismophic}},R_{\ref{afex:double_loop}} \cup R_{\ref{afex:nonismophic}})
	\end{equation*}
	 from Example~\ref{ex:nonismophic}, which is the union of $F_{\ref{afex:double_loop}} = (A_{\ref{afex:double_loop}},R_{\ref{afex:double_loop}})$ from Example \ref{ex:double_loop} and $F_{\theafexample} = (A_{\ref{afex:nonismophic}},R_{\ref{afex:nonismophic}})$ from Example \ref{ex:nonismophic}.
	We show that $a\, \simeq^\text{dbs}_F\, h$ holds by carrying out the reduction of the problem to {\problemAutomataEquivalence}.
    
    The input of {\problemDisEquivalence} is \( (F_{\text{joint}},a,h) \) and as described in Theorem~\ref{thm:DisEqToGraphWalks} the reduction to {\problemGraphPathLength} is the identity function, yielding the same tuple.
    Following Theorem~\ref{thm:maingraph}, the reduction \( \problemGraphPathLength \leq_{p} \problemAutomataEquivalence \) computes for \( (F_{\text{joint}},a,h) \) the two \( \mathbb{Q} \)-automata \( \mathcal{A}^{a}_{{\text{joint}}}=\mathcal{A}^{a}_{F_{\text{joint}}} \) and \( \mathcal{A}^{h}_{{\text{joint}}} = \mathcal{A}^{h}_{F_{\text{joint}}}  \).
    As \( a \) is the only accepting state in \( \mathcal{A}^{a}_{{\text{joint}}} \), it is sufficient to consider only the connected component of \( \mathcal{A}^{a}_{{\text{joint}}} \) which contains \( a \), which we denote by \( \mathcal{A}_{a} \). Analogously, we denote with  \( \mathcal{A}_{h} \) the restriction of \( \mathcal{A}^{h}_{{\text{joint}}} \) to the connected component which contains \( h \).
    The automata \( \mathcal{A}_{a} \) and \( \mathcal{A}_{h} \) are presented in Figure~\ref{fig:fullexampleleftright}.
    Following the line of reductions in Theorem~\ref{thm:DisEqToGraphWalks} and Theorem~\ref{thm:maingraph}, we have $a\, \simeq^\text{dbs}_F\, h$ exactly when \( \mathcal{A}_{a} \) and \( \mathcal{A}_{h} \) are equivalent.
	
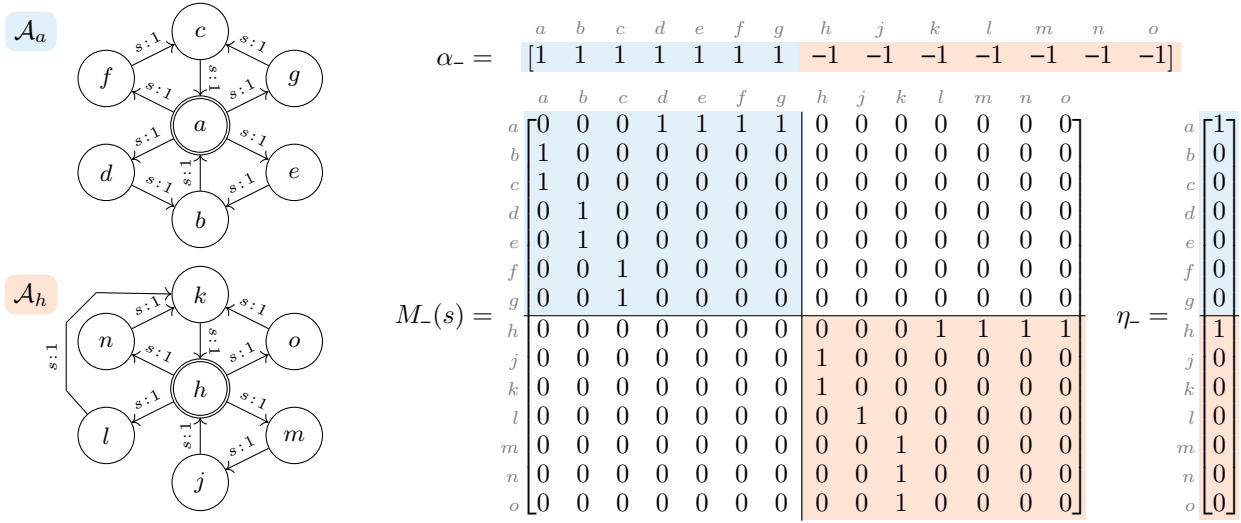
\begin{figure*}[t]
\begin{minipage}[c]{0.3\textwidth}
        \centering
        \begin{tikzpicture}[every initial by arrow/.style={initial distance=1em}]
            \tikzstyle{withsplit} = [circle, draw, minimum size= 0.75cm, inner sep=1,font=\small]
            \begin{scope}
                \node[rectangle,rounded corners,fill=CornflowerBlue!15] at (-2.25,1.25) {\( \mathcal{A}_{a} \)};
                \node[withsplit,accepting] (a) at (0,0)  {$a$\nodepart{lower} $1$};
                \node[withsplit] (b) at (0,-1.25)  {$b$\nodepart{lower}$0$};
                \node[withsplit] (c) at (0,1.25)  {$c$\nodepart{lower}$0$};
                \node[withsplit] (d) at (-1.25,-0.625)  {$d$\nodepart{lower}$0$};
                \node[withsplit] (e) at (1.25,-0.625)  {$e$\nodepart{lower}$0$};
                \node[withsplit] (f) at (-1.25,0.625)  {$f$\nodepart{lower}$0$};
                \node[withsplit] (g) at (1.25,0.625)  {$g$\nodepart{lower}$0$};

                \path[->] (a) edge node [above, sloped] {\tiny $s\!:\!1$} (e);
                \path[->] (a) edge node [above, sloped] {\tiny $s\!:\!1$} (d);
                \path[->] (a) edge node [above, sloped] {\tiny $s\!:\!1$} (g);
                \path[->] (a) edge node [above, sloped] {\tiny $s\!:\!1$} (f);
                
                \path[->] (b) edge node [above, sloped] {\tiny $s\!:\!1$} (a);
                \path[->] (c) edge node [above, sloped] {\tiny $s\!:\!1$} (a);
                
                \path[->] (d) edge node [above, sloped] {\tiny $s\!:\!1$} (b);
                \path[->] (e) edge node [above, sloped] {\tiny $s\!:\!1$} (b);
                \path[->] (f) edge node [above, sloped] {\tiny $s\!:\!1$} (c);
                \path[->] (g) edge node [above, sloped] {\tiny $s\!:\!1$} (c);
            \end{scope}
            
            \begin{scope}[yshift=-3.5cm]
                \node[rectangle,rounded corners,fill=RedOrange!15] at (-2.25,1.25) {\( \mathcal{A}_{h} \)};
                \node[withsplit,accepting] (a) at (0,0)  {$h$\nodepart{lower}$1$};
                \node[withsplit] (b) at (0,-1.25)  {$j$\nodepart{lower}$0$};
                \node[withsplit] (c) at (0,1.25)  {$k$\nodepart{lower}$0$};
                \node[withsplit] (d) at (-1.25,-0.625)  {$l$\nodepart{lower}$0$};
                \node[withsplit, initial text={\tiny $1$}] (e) at (1.25,-0.625)  {$m$\nodepart{lower}$0$};
                \node[withsplit] (f) at (-1.25,0.625)  {$n$\nodepart{lower}$0$};
                \node[withsplit] (g) at (1.25,0.625)  {$o$\nodepart{lower}$0$};

                \path[->] (a) edge node [above, sloped] {\tiny $s\!:\!1$} (e);
                \path[->] (a) edge node [above, sloped] {\tiny $s\!:\!1$} (d);
                \path[->] (a) edge node [above, sloped] {\tiny $s\!:\!1$} (g);
                \path[->] (a) edge node [above, sloped] {\tiny $s\!:\!1$} (f);
                
                \path[->] (b) edge node [above, sloped] {\tiny $s\!:\!1$} (a);
                \path[->] (c) edge node [above, sloped] {\tiny $s\!:\!1$} (a);

                \coordinate (im1) at ([xshift=-0.15cm,yshift=-0.65cm]f.west);
                \coordinate (im2) at ([xshift=-0.15cm,yshift=0.25cm]f.west);
                \coordinate (im3) at ([xshift=0.25cm,yshift=0.65cm]f.west);
                \draw[-] (d) -- (im1) edge[-] node [above, sloped] {\tiny $s\!:\!1$}  (im2);
                \draw[-] (im2) -- (im3) edge[->] (c);
                \path[->] (e) edge node [above, sloped] {\tiny $s\!:\!1$} (b);
                \path[->] (f) edge node [above, sloped] {\tiny $s\!:\!1$} (c);
                \path[->] (g) edge node [above, sloped] {\tiny $s\!:\!1$} (c);
            \end{scope}
        \end{tikzpicture}
\end{minipage}
\begin{minipage}[c]{0.69\textwidth}
        \begin{align*}        
            \alpha_{\text{\textminus}} & = \hspace{1em} \begin{bNiceArray}{cccccccccccccc}[first-row,
                code-for-first-row=\color{gray}\scriptstyle]
                \CodeBefore
                \rectanglecolor{CornflowerBlue!15}{1-1}{1-7}
                \rectanglecolor{RedOrange!15}{1-8}{1-14}
                \Body
                a & b & c & d & e & f & g & h  & j  & k  & l & m & n & o \\
                1 & 1 & 1 & 1 & 1 & 1 & 1 & \text{\textminus}1 & \text{\textminus}1 & \text{\textminus}1 & \text{\textminus}1 & \text{\textminus}1 & \text{\textminus}1 & \text{\textminus}1 
            \end{bNiceArray} &&\\
            M_{\text{\textminus}}(s) & =  \begin{bNiceArray}{ccccccc|ccccccc}[first-row,first-col,
                code-for-first-col=\color{gray}\scriptstyle,
                code-for-first-row=\color{gray}\scriptstyle]
                \CodeBefore
                \rectanglecolor{CornflowerBlue!15}{1-1}{7-7}
                \rectanglecolor{RedOrange!15}{8-8}{14-14}
                \Body
                & a & b & c & d & e & f & g & h  & j  & k  & l & m & n & o \\
                a & 0 & 0 & 0 & 1 & 1 & 1 & 1 & 0 & 0 & 0 & 0 & 0 & 0 & 0 \\
                b & 1 & 0 & 0 & 0 & 0 & 0 & 0 & 0 & 0 & 0 & 0 & 0 & 0 & 0 \\
                c & 1 & 0 & 0 & 0 & 0 & 0 & 0 & 0 & 0 & 0 & 0 & 0 & 0 & 0 \\
                d & 0 & 1 & 0 & 0 & 0 & 0 & 0 & 0 & 0 & 0 & 0 & 0 & 0 & 0 \\
                e & 0 & 1 & 0 & 0 & 0 & 0 & 0 & 0 & 0 & 0 & 0 & 0 & 0 & 0 \\
                f & 0 & 0 & 1 & 0 & 0 & 0 & 0 & 0 & 0 & 0 & 0 & 0 & 0 & 0 \\
                g & 0 & 0 & 1 & 0 & 0 & 0 & 0 & 0 & 0 & 0 & 0 & 0 & 0 & 0 \\\hline
                h & 0 & 0 & 0 & 0 & 0 & 0 & 0 & 0 & 0 & 0 & 1 & 1 & 1 & 1 \\
                j & 0 & 0 & 0 & 0 & 0 & 0 & 0 & 1 & 0 & 0 & 0 & 0 & 0 & 0 \\
                k & 0 & 0 & 0 & 0 & 0 & 0 & 0 & 1 & 0 & 0 & 0 & 0 & 0 & 0 \\
                l & 0 & 0 & 0 & 0 & 0 & 0 & 0 & 0 & 1 & 0 & 0 & 0 & 0 & 0 \\
                m & 0 & 0 & 0 & 0 & 0 & 0 & 0 & 0 & 0 & 1 & 0 & 0 & 0 & 0 \\
                n & 0 & 0 & 0 & 0 & 0 & 0 & 0 & 0 & 0 & 1 & 0 & 0 & 0 & 0 \\
                o & 0 & 0 & 0 & 0 & 0 & 0 & 0 & 0 & 0 & 1 & 0 & 0 & 0 & 0       
            \end{bNiceArray}
            \hspace{1em}
            \eta_{\text{\textminus}} = \begin{bNiceArray}{c}[first-col,
                code-for-first-col=\color{gray}\scriptstyle]
                \CodeBefore
                \rectanglecolor{CornflowerBlue!15}{1-1}{7-1}
                \rectanglecolor{RedOrange!15}{8-1}{14-1}
                \Body
                a & 1 \\
                b & 0 \\
                c & 0 \\
                d & 0 \\
                e & 0 \\
                f & 0 \\
                g & 0 \\\hline
                h & 1 \\
                j & 0 \\
                k & 0 \\
                l & 0 \\
                m & 0 \\
                n & 0 \\
                o & 0   
            \end{bNiceArray}
        \end{align*}
\end{minipage}
\caption{Automata $\mathcal{A}_{a} $ and $\mathcal{A}_{h} $ (left) and the linear representation \( \tuple{\alpha_{\text{\textminus}},M_{\text{\textminus}},\eta_{\text{\textminus}}} \) of the automaton $\mathcal{A}_{\text{\textminus}} $ (right) from Example~\ref{ex:fullexample}.\\ Left side:
    An edge with label \( s\!:\!1 \) indicates that transitioning from the source to the target when reading \( s \) has weight \( 1 \). Edges with weight zero are omitted.
    Initial weights and final weights of the states are not explicitly represented; each state has an initial weight of $1$, and each state has a final weight of $0$, except for double-circled states, which have a final weight of $1$ (\enquote{accepting state}).}
\label{fig:fullexampleleftright} \end{figure*}

Next, we decide equivalence of \( \mathcal{A}_{a} \) and \( \mathcal{A}_{h} \).
For that, we compute an automaton \( \mathcal{A}_{\text{\textminus}} \), that is empty exactly when \( \mathcal{A}_{a} \) and \( \mathcal{A}_{h} \) are equivalent.
The construction of \( \mathcal{A}_{\text{\textminus}} \) is given in Section~\ref{sec:QautomataEQ}, and the resulting linear representation \( \tuple{\alpha_{\text{\textminus}},M_{\text{\textminus}},\eta_{\text{\textminus}}} \) of the automaton \( \mathcal{A}_{\text{\textminus}} \) is given in Figure~\ref{fig:fullexampleleftright}.
Basically, \( \mathcal{A}_{\text{\textminus}} \) is the union of \( \mathcal{A}_{a} \) and \( \mathcal{A}_{h} \) where the initial values of states of \( \mathcal{A}_{\text{\textminus}} \) is $-1$ instead of $1$.
The base \( \mathcal{B}_{\text{\textminus}} = \{ b_1,b_2,b_3,b_4 \} \) of the forward space $\mathcal{F}_{\text{\textminus}} = \mathcal{F}_{\mathcal{A}_{\text{\textminus}}} = \mathrm{span}(\{ \alpha_{\text{\textminus}} \cdot M_{\text{\textminus}}(w) \mid w \in \Sigma^* \}) $ is:
\vspace{-0.5em}
\newcommand{\tm}{\text{\textminus}}
\begin{align*}
	&\begin{NiceArray}{p{1.05em}p{0.5em}p{0.5em}p{0.5em}p{0.5em}p{0.5em}p{0.5em}p{0.6em}p{0.6em}p{0.6em}p{0.6em}p{0.6em}p{0.6em}p{2em}}[first-row,
		code-for-first-row=\color{gray}\scriptsize]
		\hspace{0.5em}a & b & c & d & e & f & g & \,h  & \,j  & \,k  & \,l & \,m & \,n & \,o 
	\end{NiceArray}\\[-0.5em]
	b_1 = 
	& \begin{NiceArray}{p{1.05em}p{0.5em}p{0.5em}p{0.5em}p{0.5em}p{0.5em}p{0.5em}p{0.6em}p{0.6em}p{0.6em}p{0.6em}p{0.6em}p{0.6em}p{2em}}
		{[} 1 & 1 & 1 & 1 & 1 & 1 & 1 & \tm1 & \tm1 & \tm1 & \tm1 & \tm1 & \tm1 & \tm1 {]}
	\end{NiceArray} \\
	b_2 =& \begin{NiceArray}{p{1.05em}p{0.5em}p{0.5em}p{0.5em}p{0.5em}p{0.5em}p{0.5em}p{0.6em}p{0.6em}p{0.6em}p{0.6em}p{0.6em}p{0.6em}p{2em}}
		{[} 2 & 2 & 2 & 1 & 1 & 1 & 1 & \tm2 & \tm1 & \tm3 & \tm1 & \tm1 & \tm1 & \tm1 {]}
	\end{NiceArray} \\
	b_3 =& \begin{NiceArray}{p{1.05em}p{0.5em}p{0.5em}p{0.5em}p{0.5em}p{0.5em}p{0.5em}p{0.6em}p{0.6em}p{0.6em}p{0.6em}p{0.6em}p{0.6em}p{2em}}
		{[} 4 & 2 & 2 & 2 & 2 & 2 & 2 & \tm4 & \tm1 & \tm3 & \tm2 & \tm2 & \tm2 & \tm2 {]}
	\end{NiceArray} \\
	b_4 =& \begin{NiceArray}{p{1.05em}p{0.5em}p{0.5em}p{0.5em}p{0.5em}p{0.5em}p{0.5em}p{0.6em}p{0.6em}p{0.6em}p{0.6em}p{0.6em}p{0.6em}p{2em}}
		{[} 4 & 4 & 4 & 4 & 4 & 4 & 4 & \tm4 & \tm2 & \tm6 & \tm4 & \tm4 & \tm4 & \tm4 {]}
	\end{NiceArray}
\end{align*}
For each base $b_i$, with $i\in\{1,2,3,4\}$, the word $w_i=s^{i-1}$ is such that \( b_i=\alpha_{\text{\textminus}} \cdot M_{\text{\textminus}}(w_i) \).
Recall that \( \eta_{\text{\textminus}} \) is orthogonal to \( \mathcal{F}_{\text{\textminus}} \) exactly when for all bases \( b \in \mathcal{B}_{\text{\textminus}} \) holds \( b \cdot \eta_{\text{\textminus}} = 0 \). This is the case, as \( b \cdot \eta_{\text{\textminus}} = 0 \) holds for all bases \( b \in \mathcal{B}_{\text{\textminus}} \). Because \( \eta_{\text{\textminus}} \) is orthogonal to \( \mathcal{F}_{\text{\textminus}} \), we obtain that \( \mathcal{A}_{\text{\textminus}} \) is empty. Consequently, \( \mathcal{A}_1 \) and \( \mathcal{A}_2 \) are equivalent and we obtain \( a\, \simeq^\text{dbs}_F\, h \). %
\hfill{\KSendexample}
\end{example}

\begin{algorithm}[t]
	\KwIn{\acr{AAF} $F=(A,R)$ and arguments \( a,b \in A \).}
	\KwOut{\texttt{Yes}, if $a\, \succeq^\text{dbs}_F\, b$ holds; otherwise, \texttt{No}.}
	    \( i \leftarrow 1 \);\
        \( M \leftarrow \mathrm{AdjacencyMatrix}(F) \)\;
        \While(\tcp*[f]{Theorem~\ref{thm:boundForDisEquiv}}){\( i \leq 2|A| -1 \)}{
                \lIf{\( M[a] < M[b] \) and \( i \) is odd}{\Return{\texttt{Yes}}}
                \lIf{\( M[b] < M[a] \) and \( i \) is even}{\Return{\texttt{Yes}}}
                \lIf{\( M[b] \neq M[a] \)}{\Return{\texttt{No}}}
                \( M \leftarrow M \cdot \mathrm{AdjacencyMatrix}(F) \) \tcp*{\( M^{i+1} \)}
                \( i \leftarrow i + 1 \)\;
            }
        \Return{\texttt{Yes}}\;
            \caption{Algorithm for deciding {\problemDisStronger}.}
            \label{alg:naiveDisStronger}
\end{algorithm}
A careful reconsideration of the above sequence of reductions from {\problemDisEquivalence} to {\problemAutomataEquivalence} reveals that whenever for two arguments \( a,b \) their discussion count differs, then the smallest index in which the infinite vectors \( \Dis^F(a) \) and \( \Dis^F(b) \) are different is bound linear in the size of \( F \).
\begin{theorem}
    \label{thm:boundForDisEquiv}
    Let $F=(A,R)$ be an \acr{AAF} and let \( a,b \) be two arguments of \( F \).
 If $a\, \not\simeq^\text{dbs}_F\, b$ holds, then there is some \( i \in \mathbb{N} \)  with
 \( i \leq 2|A|-1 \) such that \( \Dis^F_i(a) \neq \Dis^F_i(b) \).
\end{theorem}
\begin{proof}
Due to Theorem~\ref{thm:DisEqToGraphWalks}, we have that $a\, \not\simeq^\text{dbs}_F\, b$ holds exactly when \( a \) and \( b \) do not agree on the number of walks. For the latter, due to  Theorem~\ref{thm:limitgraph}, the smallest index \( i \), such that \( |\Paths{a}{}{i}| \neq |\Paths{b}{}{i}| \) holds is bound by \( 2|A|-1 \).
\end{proof}
We consider Theorem~\ref{thm:boundForDisEquiv} and Theorem~\ref{thm:limitgraph} as one of the most central insights of this paper. First, because there seems to be no easy way to prove the bounds that are provided by these theorems.
Moreover, Theorem~\ref{thm:boundForDisEquiv} enables us to decide {\problemDisStronger} in polynomial time.
Such an algorithm for deciding  {\problemDisStronger} is sketched in Algorithm~\ref{alg:naiveDisStronger}.
\begin{theorem}\label{thm:strongdisPTIME}
    {\problemDisStronger} is in \textbf{PTIME}.
\end{theorem}
\begin{proof}
Matrix multiplication is known to be in polynomial time~\cite{KS_Strassen1969}. 
Hence, because $M$ is of size $|A|\times|A|$, all checks and multiplications in Algorithm~\ref{alg:naiveDisStronger} are computable in polynomial time in $|A|$. The central loop is executed at most $2|A|-1$ times, rendering the whole runtime of Algorithm~\ref{alg:naiveDisStronger} to be in polynomial time.
Correctness of Algorithm~\ref{alg:naiveDisStronger} follows from Proposition~\ref{prop:walks_matrixpotency} and Theorem~\ref{thm:boundForDisEquiv}.
\end{proof}

\begin{example}
\label{ex:strongerexample}
Let \refstepcounter{afexample} $ F_{\theafexample} = (A_{\ref{afex:strongerexample}},R_{\ref{afex:strongerexample}}) $\label{afex:strongerexample} be the \acr{AAF} with
\begin{align*}
    A_{\ref{afex:strongerexample}}& = \{a,b,c,d\}\\
    R_{\ref{afex:strongerexample}} & = \{(a,d), (b,a), (b,d), (c,a), (c,b), (c,c), (d,b) \}\ ,
\end{align*}
depicted in Figure \ref{fig:strongerexample} together with its adjacency matrix \( M \).
We show that $a\, \succeq^\text{dbs}_{F_{\ref{afex:strongerexample}}}\, b$ holds by executing Algorithm~\ref{alg:naiveDisStronger}.
\begin{figure}[t]
\centering
\begin{minipage}[c]{0.49\columnwidth}
    \centering
    \begin{tikzpicture}
        \node (b) at (0,0) [circle, draw,fill=Thistle!15,minimum size= 0.65cm] {$a$};
        \node (a) at (1.5,1.5) [circle, draw,fill=LimeGreen!15,minimum size= 0.65cm] {$b$};   
        \node (c) at (0,1.5) [circle, draw,minimum size= 0.65cm] {$c$}; 
        \node (d) at (1.5,0) [circle, draw,minimum size= 0.65cm] {$d$};

        \draw[->] (c) edge (b);
        \draw[->] (a) edge (b);
        \draw[->] (c) to [out=220,in=250,looseness=8] (c);
        \draw[->] (c) edge (a);
        \draw[->] (d) edge[bend left=12] (a);
        \draw[->] (b) edge (d);
        \draw[->] (a) edge[bend left=12] (d);
    \end{tikzpicture}
\end{minipage}
\begin{minipage}[c]{0.49\columnwidth}
    \centering
    \( M = \begin{bNiceArray}{cccc}[first-row,first-col,
        code-for-first-col=\color{gray}\scriptstyle,
        code-for-first-row=\color{gray}\scriptstyle,
        columns-width = auto]
        \CodeBefore
        \columncolor{Thistle!15}{1}
        \columncolor{LimeGreen!15}{2}
        \Body
          & a & b & c & d \\ 
        a & 0 & 0 & 0 & 1 \\
        b & 1 & 0 & 0 & 1 \\
        c & 1 & 1 & 1 & 0 \\
        d & 0 & 1 & 0 & 0 \\           
    \end{bNiceArray} \)
\end{minipage}
\caption{\acr{AAF} $F_{\ref{afex:strongerexample}}$ from Example~\ref{ex:strongerexample} and its adjacency matrix.}
\label{fig:strongerexample} \end{figure}
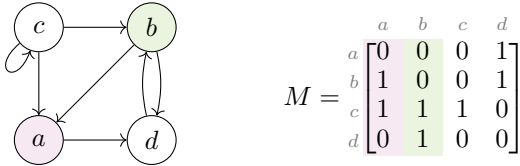
Algorithm~\ref{alg:naiveDisStronger} computes the respective matrix power \( M^i \) of the adjacency matrix succinctly.
This is because Proposition~\ref{prop:walks_matrixpotency} guarantees that \( M^i[\vertexa] \), i.e., the sum of the respective column vector of $\vertexa$,  equals \( |\Paths{\vertexa}{}{i}| \). 
The powers \( M^2 \) and \( M^3 \) are:
\begin{align*}
M^2 & = \begin{bNiceArray}{cccc}[first-col,
    code-for-first-col=\color{gray}\scriptstyle,
    first-row,code-for-first-row=\color{gray}\scriptstyle,
    columns-width = auto]
    \CodeBefore
    \columncolor{Thistle!15}{1}
    \columncolor{LimeGreen!15}{2}
    \Body
      & a & b & c & d \\ 
    a & 0 & 1 & 0 & 0 \\
    b & 0 & 1 & 0 & 1 \\
    c & 2 & 1 & 1 & 2 \\
    d & 1 & 0 & 0 & 1 \\             
\end{bNiceArray}
&
M^3 & = \begin{bNiceArray}{cccc}[first-col,
    code-for-first-col=\color{gray}\scriptstyle,
    first-row,code-for-first-row=\color{gray}\scriptstyle,
    columns-width = auto]
    \CodeBefore
    \columncolor{Thistle!15}{1}
    \columncolor{LimeGreen!15}{2}
    \Body
      & a & b & c & d \\ 
    a & 1 & 0 & 0 & 1 \\
    b & 1 & 1 & 0 & 1 \\
    c & 2 & 3 & 1 & 3 \\
    d & 0 & 1 & 0 & 1 \\            
\end{bNiceArray}        
\end{align*}
As one can check we have \( M[a]=2=M[b] \) and \( M^2[a]=3=M^2[b] \).
When the central loop of Algorithm~\ref{alg:naiveDisStronger} is executed the third time, we obtain \( M^3[a]=4 < 5=M^3[b] \) and, because $3$ is odd, Algorithm~\ref{alg:naiveDisStronger} returns \texttt{Yes}.
This corresponds to $a\, \succeq^\text{dbs}_{F_{\ref{afex:strongerexample}}}\, b$, which holds because we have:
\begin{equation*}
    \Dis^{F_{\!\ref{afex:strongerexample}}}(a) = \tuple{-2,3,-4,...} \geq^\text{lex} \tuple{-2,3,-5,...} =\Dis^{F_{\!\ref{afex:strongerexample}}}(b) 
\end{equation*}
\end{example}

\begin{algorithm}[t]
	\KwIn{\acr{AAF} $F=(A,R)$ and arguments \( a,b \in A \).}
	\KwOut{\texttt{Yes}, if $a\, \simeq^\text{dbs}_F\, b$ holds; otherwise, \texttt{No}.}
	\( i \leftarrow 1 \);\
	\( M \leftarrow \mathrm{AdjacencyMatrix}(F) \)\;
	\While(\tcp*[f]{Theorem~\ref{thm:boundForDisEquiv}}){\( i \leq 2|A| -1 \)}{
		\lIf{\( M[b] \neq M[a] \)}{\Return{\texttt{No}}}
		\( M \leftarrow M \cdot \mathrm{AdjacencyMatrix}(F) \) \tcp*{\( M^{i+1} \)}
		\( i \leftarrow i + 1 \)\;
	}
	\Return{\texttt{Yes}}\;
	\caption{Algorithm for deciding {\problemDisEquivalence}.}
	\label{alg:naiveDisEQ}
\end{algorithm}
With Theorem~\ref{thm:strongdisPTIME} established, one easily adapts Algorithm~\ref{alg:naiveDisStronger} to Algorithm~\ref{alg:naiveDisEQ} which computes {\problemDisEquivalence}.
Finally, the reduction from Section~\ref{sec:Graph2Semiring} also shows ${\problemGraphPathLength} \leq_p {\problemAutomataEquivalenceZ}$, whereby the latter is the equivalence problem of semiring automata over the integers, which is also in \textbf{PTIME} \cite{KS_BealLombardySakarovitch2005}. However, there is no analogue to Proposition~\ref{prop:complxityAutomataStuffWitness} known for \( \mathbb{Z} \) and hence, using \( \mathbb{Q} \) is vital to identify the complexity of {\problemDisStronger}.

\section{Conclusion}
\label{sec:conclusion}
We settled a computational question that had long been unanswered, namely the question of decidability and computational complexity of the discussion-based semantics in abstract argumentation. We showed that deciding whether two arguments are ranked equally or whether one argument is ranked strictly better than the other are both decidable in polynomial time. In particular, we showed that it suffices to explore walks up to length $2|A|-1$ in order to correctly classify the relation of two arguments, thus providing the means to prove correctness for algorithms such as the one from  \cite{Bonzon:2023}. Our result made use of insights from automata theory, in particular regarding the question of equivalence of automata over semirings. Thus, our result provides a means to link abstract argumentation research to graph theory and automata theory.

Although we focused our analysis on the discussion-based semantics of \citeauthor{DBLP:conf/sum/AmgoudB13a} (\citeyear{DBLP:conf/sum/AmgoudB13a}), the results of this work are relevant for a broad range of ranking-based semantics, in particular those that are defined on lengths of walks. For example, the \emph{burden-based} semantics of \cite{DBLP:conf/sum/AmgoudB13a}, the larger family of \emph{propagation-based} semantics \cite{Bonzon:2016}, and the \emph{tuples} approach \cite{Cayrol:2005} all are based on (information flowing through) walks and values determined from these walks. In particular, all these semantics highly depend on numbers of walks of different lengths. For future work, we aim to generalise the result from this paper to these approaches.

\appendix

%
%
%

%
\bibliographystyle{named}
\bibliography{ijcai26}

\end{document}